\documentclass[reqno, 11pt]{article}	
\usepackage{geometry}                
\usepackage[parfill]{parskip}    
\usepackage{graphicx,subfigure}
\usepackage{amssymb, amsmath, amsthm}
\usepackage{epstopdf}
\usepackage{multibib}
\usepackage{mathrsfs}
\usepackage[T1]{fontenc}
\usepackage[utf8]{inputenc}
\usepackage{authblk}
\newcommand{\R}{\mathbb{R}}

\newcommand{\E}{\mathbb{E}}
\newcommand{\N}{\mathbb{N}}

\newtheorem{theorem}{Theorem}[section]

\newtheorem{lemma}[theorem]{Lemma}
\newtheorem{definition}[theorem]{Definition}

\title{Bigeometric Organization of Deep Nets}
\author[1]{Alexander Cloninger}
\author[1]{Ronald R. Coifman}
\author[2]{Nicholas Downing}
\author[2]{Harlan M. Krumholz}
\affil[1]{Applied Mathematics Program, Yale University}
\affil[2]{Center for Outcomes Research and Evaluation, Yale University}

\date{}                                           

\begin{document}
\maketitle

\begin{abstract}
In this paper, we build an organization of high-dimensional datasets that cannot be cleanly embedded into a low-dimensional representation due to missing entries and a subset of the features being irrelevant to modeling  functions of interest.  Our algorithm begins by defining coarse neighborhoods of the points and defining an expected empirical function value on these neighborhoods.  We then generate new non-linear features with deep net representations tuned to model the approximate function, and re-organize the geometry of the points with respect to the new representation.  Finally, the points are locally z-scored to create an intrinsic geometric organization which is independent of the parameters of the deep net, a geometry designed to assure smoothness with respect to the empirical function.  We examine this approach on data from the Center for Medicare and Medicaid Services Hospital Quality Initiative, and generate an intrinsic low-dimensional organization of the hospitals that is smooth with respect to an expert driven function of quality.  
\end{abstract}

 \section{Introduction}
Finding low dimensional embeddings of high dimensional data is vital in understanding the organization of unsupervised data sets.   However, most embedding techniques rely on the assumption that the data set is locally Euclidean \cite{coifman2006diffusion,Saul2003,Belkin2003}.  In the case that features carry implicit weighting, some features are possibly irrelevant, and most points are missing some subset of the features, Euclidean neighborhoods can become spurious and lead to poor low dimensional representations.  

In this paper, we develop the method of expert driven functional discovery to deal with the issue of spurious neighborhoods in data sets with high dimensional contrasting features.  This allows small amounts of input and ranking from experts to propagate through the data set in a non linear, smooth fashion.  We then build a distance metric based off these opinions that learns the invariant and irrelevant features from this expert driven function.  Finally, we locally normalize this distance metric to generate a global embedding of the data into a homogeneous space.


An example to keep in mind throughout the paper, an idea we expand upon in Section \ref{hospitals}, is a data set containing publicly-reported measurements of hospital quality.  The Center for Medicare and Medicaid Services Hospital Quality Initiative reports approximately 100 different measures describing various components of the quality of care provided at Medicare-Certified hospitals across the United States.  These features range in measuring hospital processes, patient experience, safety, rates of surgical complications, and rates of various types of readmission and mortality.
There are more than $5,000$ hospitals that reported at least on measure during 2014, but only $1,614$ hospitals with 90\% measures reported.  The measures are computed quarterly, and are publicly available through the Hospital Compare website \cite{hospitalcompare}.  The high dimensional nature of these varied measures make comprehensive inferences about hospital quality impossible without summarizing statistics. 
 

Our goal is more than just learning a ranking function $f$ on the set of hospitals $X$.  We are trying to characterize the cohort of hospitals and organize the geometry of the data set, and learn a multi-dimensional embedding of the data for which the ranking function is smooth.  This gives an understanding of the data that doesn't exist with a one dimensional ranking function.  Specifically, we are looking for meta-features of the data in order to build a metric $\rho: X\times X\rightarrow \R^{+}$ that induces a small Lipschitz constant on the function $f$, as well as on features measured by CMS.  

An example of this organization is shown in Figure \ref{fig:organizeFeaturesIntro}.  The organization is generated via our algorithm of expert driven functional discovery, the details of which are found in Sections \ref{organization} and \ref{DNN}.  The colors in each image correspond to three notable CMS features: risk standardized 30 day hospital-wide readmission, patient overall rating of the hospital, and  risk standardized 30 day mortality for heart failure.  This organization successfully separates hospitals into regimes for which each feature is relatively smooth.  
\begin{figure}[h!]
\begin{tabular}{ccc}
\includegraphics[width=.3\textwidth]{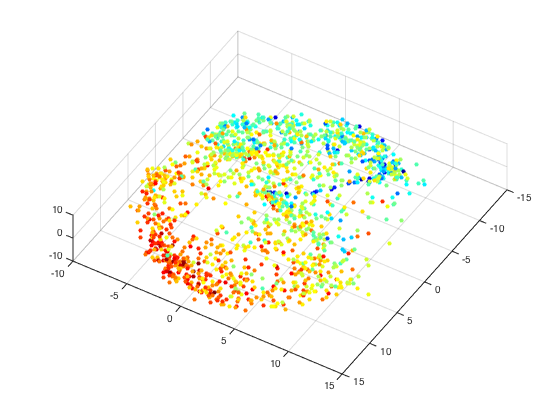} & 
\includegraphics[width=.3\textwidth]{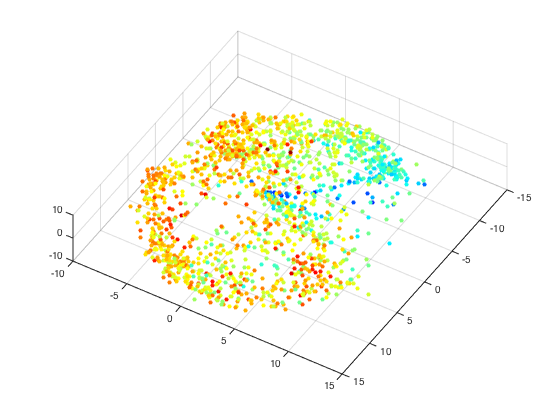} & 
\includegraphics[width=.3\textwidth]{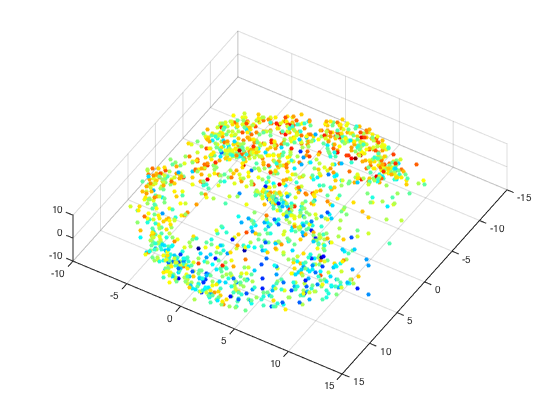} 
\end{tabular}
\caption{Organization colored by: (left)  risk standardized 30 day hospital wide readmission, (center) percent patients rating overall hospital 9 or 10 out of 10, (right)  risk standardized 30 day mortality for heart failure.  Embedding generated via bigeometric organization of deep nets.  Red is good performance, blue is bad.}\label{fig:organizeFeaturesIntro}
\end{figure}

Our organization is accomplished via a two step processing of the data.  First, we look to characterize the topology of the data set by creating a stable representation of $X$ such that neighborhoods can be differentiated at varying scales, and neighborhoods are preserved across varying parameters of the algorithm.  This is done in Sections \ref{organization} and \ref{DNN}, using a spin cycling of deep nets trained on some expert driven function.  Second, we build a metic on that topology by taking a local Mahalanobis distance on the stable neighborhoods (i.e. local z-scoring) \cite{mahalanobis,singer2008non}.  This is done in Section \ref{mahalanobis}, and guarantees that the induced metric is homogeneous.  This means that, if $U_x$ denotes the neighborhood of a point $x$, then $\rho(x,y)$ for $x,y\in U_x$ measures the same notion of distance as $\rho(x',y')$ for $x',y'\in U_{x'}$.

Figure \ref{fig:normalizedNeighborhoods} demonstrates the issue we refer to here.  The left image shows an organization of the hospitals in which neighborhoods have not been normalized.  The red points refer to ``good'' hospitals, and the blue refer to ``bad'' hospitals (in a naive sense).  The organization on the left gives no notion of the spread of the points and leaves certain questions unanswerable (i.e. are ``good'' hospitals as diverse as ``bad'' hospitals).  However, if we locally z-score the regions of the hospitals, as in the right figure, it the global notion of distance is normalized and it becomes clear that ``good'' hospitals share many more similarities amongst themselves than ``bad'' hospitals do among themselves.
 \begin{figure}
 \begin{tabular}{cc}
  \includegraphics[width=.45\textwidth]{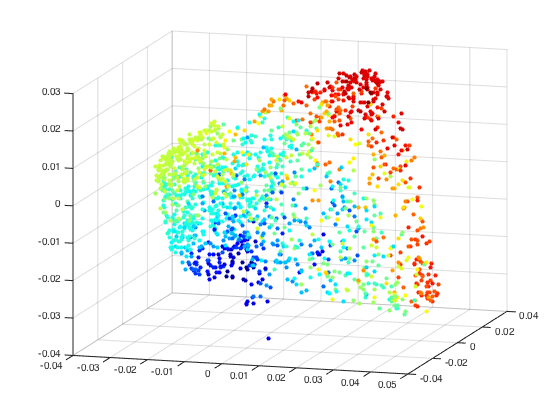} 
 \includegraphics[width=.5\textwidth]{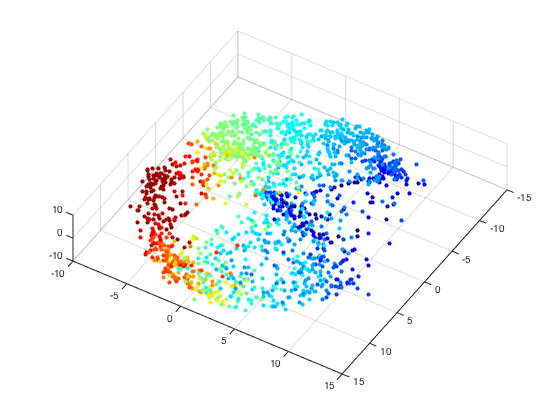} 
 \end{tabular}
 \caption{(left) unnormalized embedding, (right) z-scored embedding.}\label{fig:normalizedNeighborhoods}
 \end{figure}
 
 Also, by taking a local z-scoring of the features, we generate an organization that is dependent only on the neighborhoods $U_x$, rather than being dependent on the specific representations used.  Figure \ref{fig:mahalanobisEmbeddings} shows the organizations of the hospitals generated by algorithms with two very different parameter sets which, after z-scoring the neighborhoods, generate similar embeddings.  More details about this embedding can be found in Section \ref{initialRankings}.
 
  \begin{figure}
 \begin{tabular}{cc}
 \includegraphics[width=.5\textwidth]{diffEmbedding_smoothedStars_new.png} & 
  \includegraphics[width=.5\textwidth]{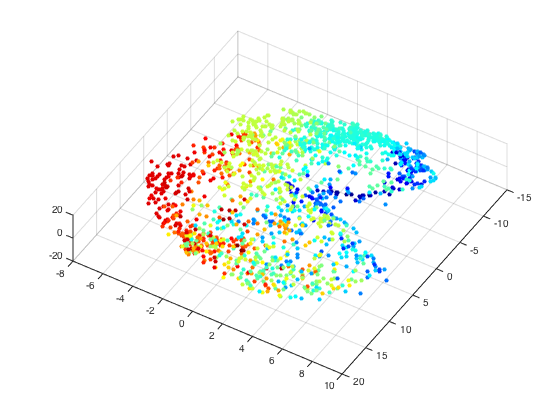}
 \end{tabular}
 \caption{Two sets of organizations with similar neighborhood structure.  The representations used to generate these embeddings are fundamentally different, with the right figure using $5\times$ as many features as the left, and with features being generated with vastly different algorithmic parameters.}\label{fig:mahalanobisEmbeddings}
 \end{figure} 
 
However, discovering the topology of the hospitals is non-trivial.  The features may have significant disagreement, and not be strongly correlated across the population.  To examine these relationships, one can consider linear correlations via principal component analysis.  The eigenvalues of the correlation matrix do not show the characteristic drop off shown in linear low dimensional data sets.  In fact, 76 of the 86 eigenvalues are above $1\%$ the size of the largest eigenvalue.  Previous medical literature has also detailed the fact that many of the features don't always correlate \cite{harlanReadMort,processOutcomes}.

For this reason, there does not exist an organization for which all features are smooth and monotonically increasing.  This is why the meta-features, and organization, must be driven by minimal external expert opinion.  This observation makes the goal of our approach three fold: develop an organization of the data that is smooth with respect to as many features as possible, build a ranking function $f$ that agrees with this organization, and minimize the amount of external input necessary to drive the system.

Our expert driven functional discovery algorithm blends the data set organization of diffusion maps and coupled partition trees with the rich set of non-linear features generated by deep learning and stacked neural nets.  We build an initial organization of the data via coupled partition trees \cite{questionnaire}, and use this partitioning to generate pseudopoints that accurately represent the span of the data at a coarse scale.  This step is explained in Section \ref{organization}.  This organization is analyzed and used as an input to a series of stacked neural nets, which learn invariant representations of the data and separate disparate clusters of points.  This step is explained in Section \ref{DNN}.  See \cite{deeplearningreview} for a review of stacked neural nets and deep learning.

It is important to note that our use of stacked neural nets is different from traditional deep learning applications.  We discuss these differences in Section \ref{diffTypeDNN}.  The purpose of using deep learning and organizing of the generated representations is to create a notion of fine neighborhoods between points; neighborhoods where the number of neighbors scales smoothly with the distance metric.

We then examine and validate our algorithm on the CMS Quality Initiative features in Section \ref{hospitals}.

 \section{Information Organization and Expert Driven Functional Discovery}\label{organization}
 \subsection{Training on Data with Full Knowledge}\label{initialQuestionnaire}
Let the data matrix be
\begin{eqnarray}\label{eq:M}
M = [v_{1}, ..., v_{N}],
\end{eqnarray}
where $v_{i}\in \R^{m}$ is a vector of observations describing the $i^{th}$ data point.  Each $v_{i}$ is allowed to have arbitrarily many missing entries.  Define supp$(v_{i}) = \{k\in\{1,...,m\} : v_{i,k} \text{ is observed}\}$.

Due to the missing entries, calculating an affinity between every two points $v_{i}$ and $v_{j}$ is not necessarily possible, given that the intersection of the supports of their known values may be small or even disjoint.  For this reason, we begin by restricting ourselves to points $v_{i_{j}}$ that have at most $\eta$ missing entries.  We shall begin by organizing the set of points $\Omega = [v_{i_{1}}, ... , v_{i_{n}}]$.

To gain an initial understanding of the geometry of $\Omega$, we consider the cosine affinity matrix $A$ where
\begin{eqnarray}
A_{j,k} = \frac{ \langle{v_{i_{j}}, v_{i_{k}}}\rangle }{\|v_{i_{j}}\| \|v_{i_{k}}\|},
\end{eqnarray}
where the inner product is calculated only on the entries in supp$(v_{i_{j}}) \cap$supp$(v_{i_{i}})$.  By definition of $\Omega$, this set contains at least $m-2\eta$ known values.

The cosine affinity matrix serves as a good starting point for learning the geometry of $\Omega$.  However, we must develop a way to extend any analysis to the full data $M$.  For this reason, we partition both the data points and the observation sets of $\Omega$.  This gives us two advantages: partitioning the data points captures the ways in which different observations may respond to different subsets of $\Omega$, and partitioning the observation sets into similar question groups gives a method for filling in the missing observations in $M$.
 
We construct a coupled geometry of $\Omega$ using the algorithm developed in \cite{questionnaire}.  The initial affinity is given by the cosine affinity matrix $A$, and the iterative procedure is updated using Earth Mover Distance \cite{earthmover}. 

Remark: Let the final affinity matrix be called $\widetilde{A}: M\times M \rightarrow [0,1]$.  Let the eigenpairs of $\widetilde{A}$ be called $\{(\lambda_i, \phi_i)\}$ with $1 = \lambda_0\ge ... \ge \lambda_N$.  Then the organization of $M$ is generated by 
\begin{eqnarray*}
\Phi^t(x) = [\lambda_1^t \phi_1(x), ..., \lambda_d^t \phi_d(x)], & x\in M,
\end{eqnarray*}
where $d$ is the dimension of the underlying manifold.

%
%
%

  \subsection{Filling in Missing Features}\label{missingFeatures}

   Let  $\mathscr{T}_{obs} $ be the hierarchical tree developed on the observations in $\R^{m}$ from Section \ref{initialQuestionnaire}.  Let the levels be $\mathscr{X}^{1}, ..., \mathscr{X}^{L}$, with the nodes for level $l$ named $\mathscr{X}_{1}^{l}, ..., \mathscr{X}_{n(l)}^{l}$.  Let $v_{i} \not\in \Omega$ be a data point with the entry $v_{i,k}$ missing.  In order to add $v_{i}$ into the geometry of $\Omega$, we must estimate the entries in $\left(\text{supp}(v_{i})\right)^{c}$ to calculate an affinity between $v_{i}$ and other points.
  
$\mathscr{T}_{obs}$ gives a tree of correlations between the observations.  This allows us to fill in $v_{i,k}$ with similar, known entries.  Find the lowest level of the tree (most strongly correlated questions) for which observation $k\in \mathscr{X}^{l}_{j}$ and $\exists m\in \mathscr{X}^{l}_{j}$ such that $v_{i,m}$ is known.  Then the estimate of $v_{i,k}$ satisfies
\begin{eqnarray}\label{eq:missingEntries}
\widetilde{v}_{i,k} = \frac{1}{|\mathscr{X}^{l}_{j}|}\sum_{m\in \mathscr{X}^{l}_{j}} v_{i,m}.
\end{eqnarray}
  
Along with an estimate of $v_{i,k}$, \eqref{eq:missingEntries} also gives a level of uncertainty for the estimate, as smaller $l$ (i.e. coarser folders) have lower correlation and give larger reconstruction error.
 
 \subsection{Expert Driven Function on the Folders}
 Let $\mathscr{T}_{points}$ be the hierarchical tree developed on the data points in $\Omega$ from Section \ref{initialQuestionnaire}.  Let the levels be $\mathscr{X}^{1}, ..., \mathscr{X}^{L}$, with the nodes for level $l$ named $\mathscr{X}_{1}^{l}, ..., \mathscr{X}_{n(l)}^{l}$.  As the partitioning becomes finer (i.e. $l$ approaches $L$), the folders contain more tightly clustered points.  This means that the distance from an point to a centroid of a folders becomes smaller as the partitioning becomes finer.  
 
Fix the level $l$ in the tree.  The centroids of these folders can be thought of as ``types'' of data points, or 
\emph{pseudopoints}.  There are two major benefits: there are a small number of pseudopoints relative to $n$ that span the entire data space, and the pseudopoints are less noisy and more robust to erroneous observations.  

These pseudopoints are the key to incorporating expert knowledge and opinion. The pseudopoints are easier and much faster to classify than individual points, as there are a small number and they are less noisy than individual points.  Also, the pseudopoints effectively synthesize the aggregate performance of multiple hospitals.  The classifications generated by experts can be varied, anything from quality rankings to discrete classes to several descriptive features or ``meta-features'' of the bins.  Specifically, the user assigns a set of classes $\mathscr{C}$ and a classification function $g:\Omega \rightarrow \mathscr{C}$ such that
\begin{eqnarray}
\forall x\in \mathscr{X}_{j}^{l}, & g(x)  = y_{j} \in \mathscr{C}.
\end{eqnarray}
This function is understood as a rough estimate, since the classification is applied to all $x\in  \mathscr{X}_{j}^{l}$ even though the class is determined only from the centroid of $\mathscr{X}_{j}^{l}$.

This step is non-traditional in unsupervised classification algorithms.  Traditionally, the clustering is done agnostic to the final goal of classification.  However, this does not allow for a second step correction of the original classification.  By allowing a user driven component that quickly examines the cluster centroids, we are able to learn an initial classification map for the points.  This provides a rough characterization of which clusters should be collapsed due to similar classification scores, which clusters should be separated further due to drastically different classification scores, and a method of determining class labels for points on the border of multiple clusters.

This function gives a rough metric $\rho:\Omega\times\Omega\rightarrow \R^{+}$ that has dependencies of the form
\begin{eqnarray}\label{eq:rho}
\rho(x,y) = f\left( x - y; g(x),g(y) \right).  & f:\R^{m} \times \mathscr{C}\times \mathscr{C} \rightarrow \R^{+}.
\end{eqnarray}
This metric needs to satisfy two main properties:
\begin{enumerate}
\item[1.] $\exists \delta_{0}$ such that $\rho(x,y) < \epsilon_{dist}$  if $\|x-y\|<\delta$ for $\delta < \delta_{0}$ and some norm $\|\cdot\|$, and 
\item[2.] $\rho(x,y) > \epsilon_{class}$ if $g(x) \neq g(y)$.
\end{enumerate} 
A metric that satisfies these two properties naturally relearns the most important features for preserving clusters while simultaneously incorporating expert knowledge to collapses non-relevant features.  We learn this function using neural nets, as described in Section \ref{DNN}.

 \section{Deep Learning to Form Meta-Features}\label{DNN}
 There are entire classes of functions that approximate the behavior of $\rho$ in \eqref{eq:rho}.  For our algorithm, we use neural nets for several reasons.  First, the weight vectors on the first layer of a neural net have clear, physical interpretation, as the weight matrix can be thought of as a non-linear version of the weight vectors from principal component analysis.  Second, current literature on neural nets suggest a need for incredibly large datasets to develop meaningful features on unsupervised data.  Our algorithm provides a way to turn an unlabeled data set into a semi-supervised algorithm, and incorporates this supervision into the nodes of the neural net.  This supervision appears to reduce the number of training points necessary to generate a non-trivial organization.  Past literature has used radial basis functions as the non-linear activation for deep learning\cite{kerneldeeplearning} , as well as for analysis of the structure of deep net representations \cite{mullerdeeplearning}.

 \subsection{Neural Nets with Back-Propagation of Rankings}
For our algorithm, we build a 2 layer stacked autoencoder with a sigmoid activation function, with a regression layer on top.  The hidden layers are defined as 
\begin{eqnarray*}
h^{(l)}(x) = \sigma\left(b^{(l)} + W^{(1)}h^{(l-1)}(x)\right),
\end{eqnarray*}
with $\sigma:\R\rightarrow [0,1]$ being a sigmoid function applied element-wise, and $h^{(0)}(x) = x$.  The output function $f(x)$ is logistic regression of these activation units 
\begin{eqnarray*}
f(x) = \sigma(b^{(3)} + V h^{(2)}(x)).
\end{eqnarray*}
The reconstruction cost function for training our net is an $L^{2}$ reconstruction error
\begin{eqnarray}\label{eq:cost}
C = \frac{1}{n} \sum_{i=1}^{n}\|g(x_{i}) - f(x_{i})\|^{2}.
\end{eqnarray}
The overall loss function we minimize, which combines the reconstruction cost with several bounds on the weights, is
\begin{eqnarray*}
L = \frac{1}{n} \sum_{i=1}^{n}\|g(x_{i}) - f(x_{i})\|^{2} + \mu \sum_l \sum_{i,j} \left( W_{i,j}^{(l)} \right)^2 
\end{eqnarray*}
Note that we rescale $g$ if it takes values outside $[0,1]$.  We then backpropagate the error by calculating $\frac{\delta L}{\delta w^{(l)}_{i,j}}$ and $\frac{\delta L}{\delta w^{(l)}_{i}}$ and adjusting the weights and bias accordingly.  See \cite{rojasneuralnet} for a full description of the algorithm.

\begin{definition}
The deep neural net metric on $\Omega$ with respect to an external function $f$ is defined as
\begin{eqnarray*}
\rho_{DNN}(x,y) = \|h^{(1)}(x) - h^{(1)}(y) \|.
\end{eqnarray*}
\end{definition}

\begin{lemma}\label{lemma:lipschitz}
A deep neural net with a logistic regression on top generates a metric $\rho_{DNN}$ that satisfies Condition 1 from \eqref{eq:rho} with a Lipschitz constant of $\|W_{1}\|/4$ with respect to Euclidean distance.  The output function $f$ also has a Lipschitz constant of  $\|W_{1}\|\|W_{2}\|\|V\|/64$ with respect to Euclidean distance.
\end{lemma}
\begin{proof}
Let $\sigma(x) = \frac{1}{1 + e^{-x}}$.  Then $\frac{d\sigma}{d x} = \sigma(x) (1-\sigma(x)) \le \frac{1}{4}$.  By the mean value theorem, $|\sigma(a) - \sigma(b)| = |a - b| |\frac{d\sigma}{d x}(z)|\le \frac{|a - b|}{4}$.  Then 
\begin{eqnarray*}
\|f(x) - f(y)\|_{2} &\le& \|V\| \|h^{(2)}(x) - h^{(2)}(y)\|/4 \\
&\le& \|V\| \|W^{(2)}h^{(1)}(x) - W^{(2)}h^{(1)}(y)\| / 16 \\
& \le & \|V\| \|W^{(2)}\| \|h^{(1)}(x) - h^{(1)}(y)\| / 16\\
&\le& \|V\| \|W^{(2)}\| \|W^{(1)}\| \|x - y\|/64.
\end{eqnarray*}

The same argument applies for $\rho_{DNN}$.

\end{proof}

\begin{lemma}
A two layer neural net with a logistic regression on top creates a function $f$ which satisfies a variant of Condition 2 from \eqref{eq:rho}, namely that
\begin{eqnarray}\label{eq:expectedDiff}
\E_{\neq} \left(\|f(x) - f(y)\|^{2}\right)&\ge& \E_{\neq} \left(\|g(x) - g(y)\|^{2}\right) - 2 \left(\frac{\max_{i\in\mathscr{C}}S_{i} \cdot n}{S}\right)C,
\end{eqnarray}
where $S = \#\{(x,y)\in\Omega\times\Omega : g(x) \neq g(y)\}$, $S_{i} = \#\{y\in\Omega : g(y) \neq i\}$, and $\E_{\neq}$ is the expected value over the set $S$.

Moreover, the deep neural net $\rho_{DNN}$ generated also satisfies
\begin{eqnarray}\label{eq:expectedDiff}
\E_{\neq} \left(\|f(x) - f(y)\|^{2}\right) \ge \left( \frac{1}{\prod_{i=2}^{L} \|W_i\| } \right)  \left[ \E_{\neq} \left(\|g(x) - g(y)\|^{2}\right) - 2 \left(\frac{\max_{i\in\mathscr{C}}S_{i} \cdot n}{S}\right)C \right].
\end{eqnarray}
\end{lemma}
\begin{proof}
We have 
\begin{eqnarray*}
\|f(x) - f(y)\|^{2} &=& \|f(x) - g(x) - f(y) + g(y) + g(x) - g(y)\|^{2} \\
&\ge& \|g(x) - g(y)\|^{2} - \left(\|f(x)-g(x)\|^{2} + \|f(y)-g(y)\|^{2}\right).
\end{eqnarray*}
Unfortunately, because \eqref{eq:cost} is a global minimization, we cannot say anything meaningful about the difference for individual points.  However, we do have
\begin{eqnarray*}
\sum_{g(x)\neq g(y)}\|f(x) - f(y)\|^{2} &\ge& \sum_{g(x)\neq g(y)}\|g(x) - g(y)\|^{2} - \sum_{g(x)\neq g(y)}\left(\|f(x) - g(x)\|^{2} + \|f(y)-g(y)\|^{2}\right)\\
&=& \sum_{g(x)\neq g(y)}\|g(x) - g(y)\| - 2\sum_{x\in\Omega} \#\{y : g(x) \neq g(y)\} \cdot \|f(x) - g(x)\|^{2} \\
&\ge&  \sum_{g(x)\neq g(y)}\|g(x) - g(y)\| - 2\left(\max_{i\in\mathscr{C}}\#\{y : g(y)\neq i\}\right) \cdot nC,
\end{eqnarray*}
where $\#\{y : g(y)\neq i\}$ denotes the number of elements in this set.  Let $S = \#\{(x,y)\in\Omega\times\Omega : g(x) \neq g(y)\}$ and $S_{i} = \#\{y\in\Omega : g(y) \neq i\}$.  Then 
\begin{eqnarray*}
\E_{\neq} \left(\|f(x) - f(y)\|^{2}\right)&\ge& \E_{\neq} \left(\|g(x) - g(y)\|^{2}\right) - 2 \left(\frac{\max_{i\in\mathscr{C}}S_{i} \cdot n}{S}\right)C.
\end{eqnarray*}
This means that by minimizing $C$, we are forcing the separation of points with different initial ranking to be as large as possible.  This makes enforces Condition 2 of \eqref{eq:rho} in the aggregate over all such points.

The scaling of $\frac{1}{\prod_{i=2}^{L} \|W_i\|}$ for $\rho_{DNN}$ is a simple application of Lemma \ref{lemma:lipschitz}.
 \end{proof}
 
 \subsection{Heat Kernel Defined by $\rho_{DNN}$}\label{manyNets}
 
 The weights generated by the neural net represent ``meta-features'' formed from the features on $\Omega$.  Each hidden node generates important linear and non-linear combinations of the data, and contains much richer information than a single question or average over a few questions. 
 
One downside of traditional deep learning is that neural nets only account for the geometry of the observations.  They ignore the geometry and clustering of the data points, opting to treat every point equally by using a simple mean of all points in the cost function \eqref{eq:cost}. This can easily miss outliers and be heavily influenced by large clusters, especially with a small number of training points.
 
 The questionnaire from Section \ref{initialQuestionnaire} organizes both the observations and the data points, though the meta-features from the questionnaire are simply averages of similar features as in \eqref{eq:missingEntries}.  However, the expert ranking assigned to each bin reflects a rough initial geometry onto the points to be learned by the neural net.
  
 The back propagation of the expert driven function is essential to building significant weight vectors in this regime of small datasets.  When the ratio of the number of data points to the dimension of the features is relatively small, there are not enough training points to learn the connections between all the features without considering the points themselves.  The back propagation of the classification function generated from the questionnaire is a way to enforce the initial geometry of the data on the weight vectors.  The hospital ranking example in Section \ref{hospitals} demonstrates the need for back propagation, and Figure \ref{fig:saeNoBackprop} demonstrates the fact that the features from a simple SDAE are not sufficient for separating data points.
  
 For our algorithm, the SDAE is trained on $\Omega$, the subset of points with ``full'' collection of features.  This is to avoid training on reconstructed features which are subject to reconstruction error from Section \ref{missingFeatures}.  

Another problem with neural nets is that they can be highly unstable under parameter selection (or even random initialization).  Two identical iterations can lead to completely different weights set.  Along with that, back propagation can force points into isolated corners of the cube in $[0,1]^{k}$.

For this reason, we rerun the neural net $K$ times with varied random seeds, number of hidden layers, sparsity parameters, and dropout percentages.  After $K$ iterations, we build the new set of features on points as $\Omega^*(x) = [h^{(1)}_1(x), ..., h^{(1)}_{K}(x)]$.  This defines an adjacency matrix on $A$ with affinity defined between two points as 
\begin{eqnarray}\label{eq:DNNAffinity}
A(x,y) = e^{-\|\Omega^*(x) - \Omega^*(y)\|^{2}/\epsilon}.
\end{eqnarray}  
Along with that, the final ranking function on $M$ comes from $f(x) = \frac{1}{K}\sum_{i=1}^{K} f_{i}(x)$.  Note that $\|\Omega^*(x) - \Omega^*(y)\|^2 = \sum_{i=1}^K \rho_{DNN, i}(x,y)^2.$

The expert driven heat kernel defined in \eqref{eq:DNNAffinity} generates an embedding $\Phi:\Omega\rightarrow\R^d$ via the eigenvectors $\Phi^{t}(x) = [\lambda_1^t \phi_1(x), ... , \lambda_1^t \phi_d(x)]$.  
 
 For each neural net $h_{i}$, we keep the number of hidden layers small relative to the dimension of the data.  This keeps the net from overfitting the data to the initial organization function $g$.

\subsection{Standardizing Distances to Build an Intrinsic Embedding}\label{mahalanobis}

While this generates a global embedding based off local geometry, it does not necessarily generate a homogenous space.  In other words, $\|\Phi^t(x) - \Phi^t(y)\| = \|\Phi^t(x') - \Phi^t(y')\|$ does not necessarily guarantee that $x$ and $y$ differ by same amount as $x'$ and $y'$.  This is because $\Omega^*(x)$ and $\Omega^*(y)$ may differ in a large number of deep net features, whereas $\Omega^*(x')$ and $\Omega^*(y')$ may only differ in one or two features (though those features may be incredibly important for differentiation).  

For this reason, we must consider a local z-score of the regions of the data.  For each point $\Phi^t(x)$, there exists a mean and covariance matrix within a neighborhood $U_x$ about $x$ such that
\begin{eqnarray*}
\mu_x &=& \frac{1}{|U_x|} \sum_{z\in U_x} \Phi^t(z), \\ 
\Sigma_x &=& \frac{1}{|U_x|} \sum_{z\in U_x} \left( \Phi^t(z) - \mu_x \right)^\intercal \left( \Phi^t(z) - \mu_x \right).
\end{eqnarray*}
This generates a new whitened distance metric
\begin{eqnarray}
d_t(x,y) = \frac{1}{2} \left[(\Phi^t(x) - \mu_x) - (\Phi^t(y) - \mu_y)\right]^\intercal \left( \Sigma_x^{\dagger} + \Sigma_y^{\dagger} \right) \left[(\Phi^t(x) - \mu_x) - (\Phi^t(y) - \mu_y)\right],
\end{eqnarray}
where $\Sigma^{\dagger}$ is the Penrose-Moore pseudoinverse of the covariance matrix.
 
 One can generate a final, locally standardized representation of the data via the diffusion kernel
 \begin{eqnarray*}
 W(x,y) = e^{-d_t(x,y) / \sigma},
 \end{eqnarray*}
 and the low frequency eigenvalues/eigenvectors of $W$, which we call $\{(s_i,\psi_i)\}$.  The final representation is denoted 
 \begin{eqnarray}\label{eq:standardizedEmbedding}
 \Phi_{std}^t(x) = [s_i^t \psi_1(x), ... ,s_i^t \psi_d(x)].
 \end{eqnarray}

 \subsection{Extension to New Points}\label{refSetEmbedding}
 
 Once meta-features have been built, extending the embedding to the rest of $M$ can be done via an asymmetric affinity kernel, as described in \cite{refsethaddad}.  For each $x\not\in \Omega$, one can calculate $\Omega^*(x) = [h^{(1)}_{i}(x)]_{i=1}^K$ and $f_{i}(x)$ with the weights generated in Section \ref{manyNets}. This defines an affinity between $x$ and $y\in\Omega$ via 
 \begin{eqnarray*}
a(x,y) = e^{-\|\Omega^*(x) - \Omega^*(y)\|^{2}/\epsilon}.
 \end{eqnarray*}
 The affinity matrix is $a\in\R^{N\times n}$.  \cite{refsethaddad} then shows that the extension of the eigenvectors $\phi_i$ of $A$ from \eqref{eq:DNNAffinity} to the rest of $M$ is given by 
 \begin{eqnarray*}
 \widetilde{\phi}_i = \frac{1}{\lambda_i^{1/2}} \widetilde{A} \phi_i,
 \end{eqnarray*}
 where $\widetilde{A}$ is the row stochastic normalization of $a$.

 Similarly, once the embedding $\Phi$ is extended to all points in $M$, we can extend the final, standardized representation to the rest of $M$ as well using the affinity matrix
 \begin{eqnarray*}
 b(x,y) = exp\{ -\frac{1}{2} \left[(\Phi^t(x) - \mu_y) - (\Phi^t(y) - \mu_y)\right]^\intercal  \Sigma_y^{\dagger} \left[(\Phi^t(x) - \mu_y) - (\Phi^t(y) - \mu_y)\right] /\sigma \},
 \end{eqnarray*}
 for $x\in M$ and $y\in\Omega$.  Once again, the eigenvectors $\psi_i$ in \eqref{eq:standardizedEmbedding} can be extended similarly via 
 \begin{eqnarray*}
 \widetilde{\psi}_i = \frac{1}{s_i^{1/2}} \widetilde{B} \psi_i,
 \end{eqnarray*}
 where $\widetilde{B}$ is the row stochastic normalized version of $b$. 

 In this way, $\Phi_{std}^t$ can be defined on all points $M$ after analysis on the reference set $\Omega$.

 \subsection{Different Approach to Deep Learning}\label{diffTypeDNN}
Our algorithm, as we will discuss in detail in Section \ref{DNN}, uses the meta-features from stacked neural nets in a way not commonly considered in literature.  Most algorithms use back propagation of a function to fine-tune weights matrices and improve the accuracy of the one dimensional classification function.  However, in our algorithm, the purpose of the back propagation is not to improve classification accuracy, but instead to organize the data in such a way that is smooth relative to the classification function.  In fact, we are most interested in the level sets of the classification function and understanding the organization of these level sets.

At the same time, we are not building an auto-encoder that pools redundant and correlated features in an attempt to build an accurate, lossy compression (or expansion) of the data.   Due to the high level of disagreement among the features, non-trivial features generated from an auto-encoder are effectively noise, as we see in Section \ref{DNNStars}.  This is the motivation behind propagating an external notion of ``quality''.

 \section{Expert Driven Embeddings for Hospital Rankings}\label{hospitals}
 \subsection{Hospital Quality Ranking}
 
 For preprocessing, every feature is mean centered and standard deviation normalized.  Also, some of the features are posed in such a way that low scores are good, and others posed such that low scores are bad.  For this reason, we ``de-polarize'' the features using principal component analysis.  Let $M=[v_{1}, ... , v_{N}]$ be the data matrix.  Then let $U$ be the largest eigenvector of the covariance matrix
\begin{displaymath}
\left[
     \begin{array}{c}
       M \\
       -M
     \end{array}
   \right] \times
   \left[
     \begin{array}{cc}
       M^{*} &
       -M^{*}
     \end{array}
   \right],
   \end{displaymath} 
and take the half of the features $i$ for which $U_{i} > 0$.  This makes scores above the mean ``good'' regardless of whether the feature is posed as a positive or negative question.

We begin by building a questionnaire on the hospitals.  
Our analysis focuses on the 2014 CMS measures.
We put a $2\times$ weight on mortality features and $1.5\times$ weight on readmission features due to importance, because the outcome measures describe the tangible results of a hospitalization are particularly important for patients.  The questionnaire learns the relationships between the hospitals, as well as the relationships between the different features.  In doing this, we are able to build a partition tree of the hospitals, in which hospitals in the same node of the tree are more similar than hospitals in different nodes on the same level of the tree.  As a side note, the weighting on mortality features only guarantees that the intra-bin variance of the mortality features is fairly low.

For the ranking in this example, we use the $5^{th}$ layer of a dyadic questionnaire tree, which gives $32$ bins and pseudohospitals.  Experts on CMS quality measures rank these pseudohospitals on multiple criteria and assign a quality score between $1$ to $10$ to each pseudohospitals.  We use an average of $50$ nodes per layer of the neural net.  Also, we average the results of the neural nets over $100$ trials.  We shall refer to the final averaged ranking as the deep neural net ranking (DNN ranking) to avoid confusion.  

 \subsection{Two Step Embedding of Hospitals}\label{DNNStars}
%
  
 Figure \ref{fig:subsetEmbedding} shows the embedding of the subset of 
hospitals used in training the neural net.  It is colored by the quality function assigned to those hospitals.  The various prongs of the embedding are explained in Section \ref{prongsOfEmbedding}.  Figure \ref{fig:fullEmbedding} shows the embedding of the full set of hospitals.  This was generated via the algorithm in Section \ref{refSetEmbedding}.

  \begin{figure}[h!]
 \begin{center}
 \includegraphics[width=.5\textwidth]{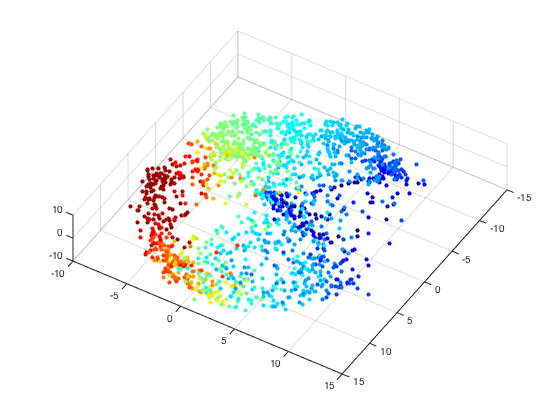}
 \end{center}
 \caption{Embedding of hospitals with at most $\tau$ missing entries.  Red corresponds to top quality, blue to bottom quality.}\label{fig:subsetEmbedding}
 \end{figure}

 \begin{figure}[h!]
 \begin{center}
 \includegraphics[width=.5\textwidth]{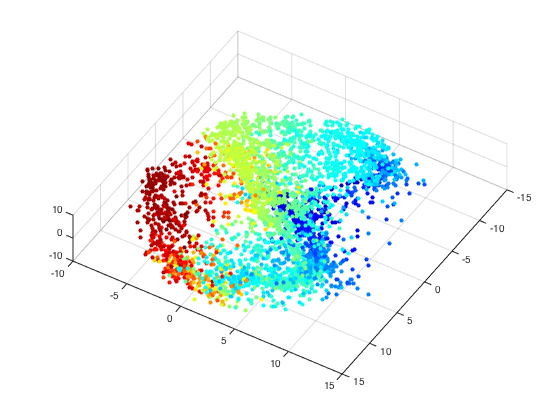}
 \end{center}
 \caption{Embedding of full set of hospitals.  Red corresponds to top quality, blue to bottom quality.}\label{fig:fullEmbedding}
 \end{figure}
 
 Finally, while the goal of the hospital organization is to determine neighborhoods of similar hospitals, it is also necessary for all hospitals in a shared neighborhood to share a common quality rating.  Figure \ref{fig:starNeighbors} plots the quality function assigned to the hospitals against the weighted average quality function of its neighborhood, where the weights come from the normalized affinities between the given hospitals and its neighbors.  The strong collinearity demonstrates that the assigned quality function is consistent within neighborhoods of similar hospitals.
 
  \begin{figure}[h!]
 \begin{center}
 \includegraphics[width=.5\textwidth]{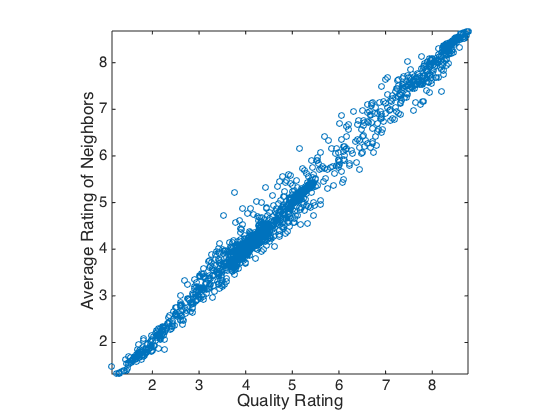}
 \end{center}
 \caption{Quality function assigned  for hospital versus average quality across the neighborhood in Figure \ref{fig:subsetEmbedding}.}\label{fig:starNeighbors}
 \end{figure}
 
 To demonstrate that the expert input back propagation is necessary for a viable ranking and affinity, we include Figure \ref{fig:saeNoBackprop}.  Here, we build the same diffusion map embedding, but on the features of the autoencoder before back propagation of the expert input function.  Due to the small number of data points relative to the number of features, an untuned autoencoder fails to form relevant meta-features for the hospitals.
 
 \begin{figure}[h!]
 \begin{center}
 \includegraphics[width=.5\textwidth]{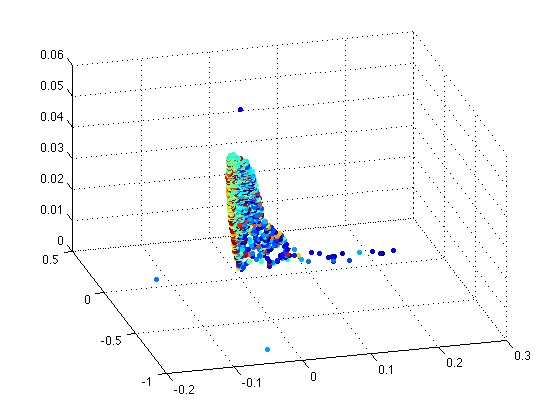}
 \end{center}
 \caption{Embedding of $\Omega$ without back propagation of expert driven hospitals rankings.  Coloring comes from quality rating function.  Notice that without back propagation the embedding is effectively noise compared to hospital quality.}\label{fig:saeNoBackprop}
 \end{figure}

 \subsection{Internal Validation}\label{interalValidation}
 The problem with validating this composite rating of hospital quality is the lack of external ground truth.  For this reason, we must rely on internal validation mechanisms to demonstrate consistency.  There are two qualities necessary for a ranking mechanism: the affinity between hospitals cannot disagree too drastically with the original features, and similar hospitals must be of similar quality.

 For the rest of the section, we refer to several simple meta-features for comparison and validation.  Row sum ranking refers to a simple ranking function 
 \begin{eqnarray*}
 f_{row-sum}(x) = \sum_{i} x_{i},
 \end{eqnarray*}
 where $\{x_{i}\}$ are the normalized and depolarized features for a given hospital $x$.  Average process, survey, and outcome features are the same as the row sum score, but restricted only to features in the given category.  NNLS weighted ranking refers to a function
 \begin{eqnarray*}
  f_{NNLS}(x) = \sum_{i} w_{i} x_{i},
 \end{eqnarray*}
 where the weights are calculated using non-negative least squares across the normalized and depolarized features, with the DNN ranking as the dependent variable.  NNLS weighted process, survey, and outcome features are the same as the NNLS weighted score, but restricted only to features in the given category.   NNLS was used for calculating the weights to avoid overfitting the DNN ranking by using negative weights.
 
 Also, our focus (unless otherwise stated) is on the subset of hospitals $\Omega$ that have more than $90\%$ features reported.  This is because a number of our validation steps use the row sum ranking and NNLS weighted ranking, which are more accurate when calculated over non-missing entries.  
  
 \subsubsection{DNN Satisfies Conditions for $\rho$}
 
Figure \ref{fig:euclideanDistanceDeviation} verifies that the back propagation neural net satisfies Condition 1 of $\rho$ from \eqref{eq:rho}.  The plot shows the ranking of each hospital plotted against the average of its ten nearest neighbors under Euclidean distance between hospital profiles.  The fact that the average strongly correlates with the original quality rating shows that the embeddings of the hospitals remain close if they are close under Euclidean distance.

  \begin{figure}[h!]
 \begin{center}
 \includegraphics[width=.5\textwidth]{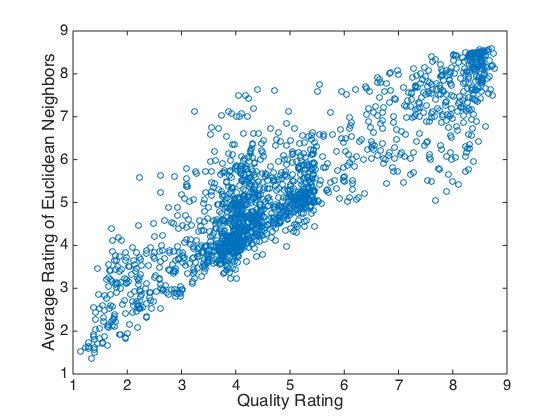}
 \end{center}
 \caption{Quality function for hospital versus average quality function across closest Euclidean points.  Demonstrates first condition necessary for $\rho$ from \eqref{eq:rho}.}\label{fig:euclideanDistanceDeviation}
 \end{figure}

Figure \ref{fig:histAffinity} shows that the back propagation neural net satisfies Condition 2 of $\rho$ from \eqref{eq:rho}.  The histogram on the left shows the DNN affinity between hospitals with different initial quality ratings (i.e. when $g(x) \neq g(y)$).  To satisfy Condition 2, $g(x) \neq g(y) \implies \rho(x,y)>\epsilon_{class}$ (i.e. $A(x,y)<1-\epsilon$).  Also, for the histogram on the right we define 
\begin{eqnarray*}
P_{\neq}(t) = \Pr \left(A(x,y)>t : g(x)\neq g(y)\right), & P_{=}(t) = \Pr\left(A(x,y)>t : g(x) = g(y)\right).
\end{eqnarray*}
The histogram on the right shows $\frac{P_{\neq}(t)}{P_{=}(t)}$.  The smaller the ratio as $t$ approaches $1$, the stronger the influence of the initial ranking on the final affinity.
 
 \begin{figure}[h!]
 \begin{tabular}{cc}
 \includegraphics[width=.5\textwidth]{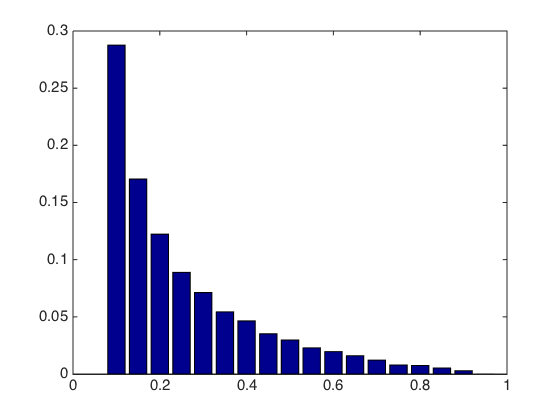} & 
  \includegraphics[width=.5\textwidth]{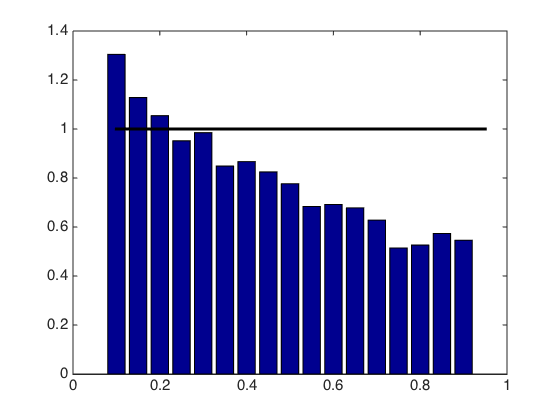}
 \end{tabular}
 \caption{(left) Histogram of DNN affinity between two points with different initial quality ratings.  
 (right) Normalized histogram of DNN affinity between two points with different initial quality ratings divided by affinity between two points with same initial quality rating.  If initial ranking was unimportant, bar graph would be concentrated around 1.  Both images demonstrate second condition necessary for $\rho$ from \eqref{eq:rho}.  Note that, for scaling purposes, all values less than $.1$ have been removed from counting.}\label{fig:histAffinity}
 \end{figure}
 
 Moreover, we can compare the contraction guaranteed by \eqref{eq:expectedDiff}.  For the hospital ratings,
 \begin{align*}
 \E_{\neq}\left(\|f(x) - f(y)\|^{2}\right) = 2.44, &\text{ } \E_{\neq}\left(\|g(x) - g(y)\|^{2}\right) = 3.36,  & \frac{\max_{i\in\mathscr{C}}S_{i} \cdot n}{S} = 1.22, & \text{ }C = 0.9959,
 \end{align*}
 which makes the right hand side of  \eqref{eq:expectedDiff} equal $2.15$.


 
 Figure \ref{fig:cumsumAffinity} indicates that the DNN metric from \eqref{eq:rho} gives a better notion of small neighborhoods than a simple Euclidean metric.  Each metric defines a transition probability $P(x,y)$.  For each point $x$, the plot finds 
 \begin{equation}\label{eq:smallNeighborhood}
\begin{aligned}
& \underset{I\subset\Omega}{\text{minimize}}
& & \# I \\
& \text{subject to}
& & \sum_{y\in I}P(x,y) \ge \frac{1}{2}.
\end{aligned}
\end{equation}
It is important for \eqref{eq:smallNeighborhood} to be small, as that implies the metric generates tightly clustered neighborhoods.  As one can see from Figure \ref{fig:cumsumAffinity}, the DNN metric creates much more tightly clustered neighborhoods than a Euclidean metric.  Figure \ref{fig:cumSumAffinityVarySigma} gives summary statistics of these neighborhoods for varying diffusion scales.

  \begin{figure}[h!]
 \begin{center}
 \includegraphics[width=.5\textwidth]{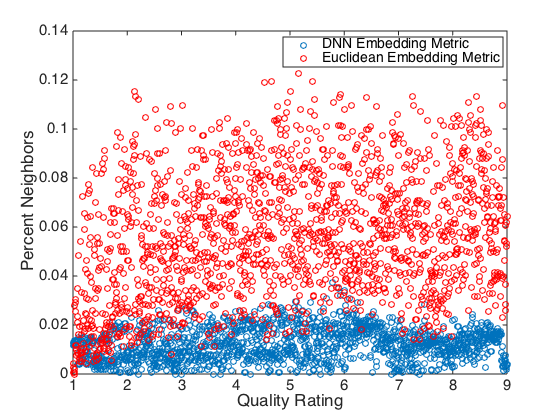} 
 \end{center}
 \caption{Percent of points necessary to sum half the total transition probability.  Metrics are generated in same fashion: $K(x,y) = e^{-\|x - y\|^2/\sigma^2}$.  Here $\sigma = \frac{1}{N}\sum_x \|x - y_x\|$, where $y_x$ is the $10^{th}$ closest neighbor of $x$.}\label{fig:cumsumAffinity}
 \end{figure}
 
   \begin{figure}[h!]
 \begin{tabular}{cc}
 \includegraphics[width=.5\textwidth]{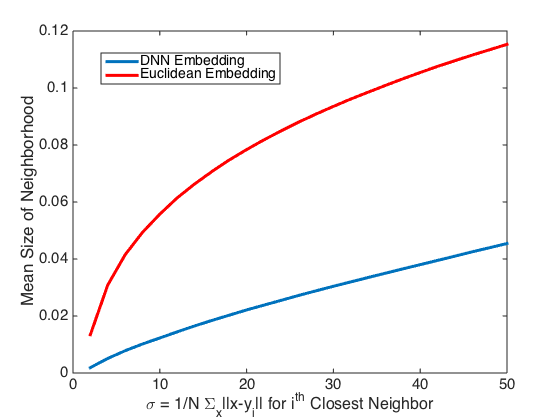} &
 \includegraphics[width=.5\textwidth]{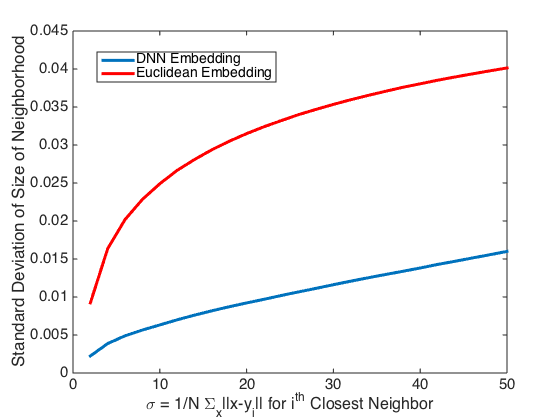}
 \end{tabular}
 \caption{Metrics are generated in same fashion: $K(x,y) = e^{-\|x - y\|^2/\sigma_i^2}$.  Here $\sigma_i = \frac{1}{N}\sum_x \|x - y_x\|$, where $y_x$ is the $i^{th}$ closest neighbor of $x$.  For both DNN embedding and Euclidean embedding metrics, the plots are (left) the average percentage of neighbors needed to contain half the transition probability from a given point to its neighborhood, and (right) the standard deviation of the number of neighbors needed to contain half the transition probability from a given point to its neighborhood.}\label{fig:cumSumAffinityVarySigma}
 \end{figure}

 Another positive characteristic of our DNN embedding is the reduced local dimension of the data.  Figure \ref{fig:eigenvals} plots the eigenvalues $\{S_{i}^{t}\}_{i=0}^{n}$ of the diffusion kernels $A_{t}$, where the Markov chain $A_{t}$ describes either comes from the Euclidean metric or the DNN metric.  These eigenvalues give us information about the intrinsic dimension of the data, as small eigenvalues do not contribute to the overall diffusion.  To compare the eigenvalues across different diffusion kernels, we normalize the eigenvalues by setting $t = \frac{1}{1-S_{1}^{1}}$, as this is the average time it takes to diffuse across the system.  Cutting off the eigenvalues to determine the dimension is fairly arbitrary without a distinct drop off, but an accepted heuristic is to set the cutoff at 
 \begin{eqnarray*}
 \text{dim}(A_{t}) = \max \{d \in \N : S_{d} > 0.01\}.
 \end{eqnarray*}
 With this definition of dimension, dim$(A_{t}) = 7$ for the DNN metric embedding, whereas dim$(A_{t}) = 14$ for the Euclidean metric embedding.
 
 \begin{figure}[h!]
 \begin{center}
 \includegraphics[width=.5\textwidth]{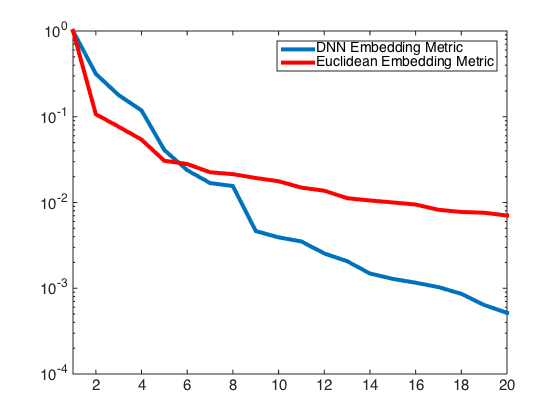}
 \end{center}
 \caption{Eigenvalues $\{S_{i}^{t}\}_{i=0}^{n}$ of the embedding of the heat kernel $A_{t}$ generated from the Euclidean metric (blue) and DNN metric (red).  Each eigenvalue set is normalized by setting $t = \frac{1}{1-S_{1}^{1}}$. }\label{fig:eigenvals}
 \end{figure}

\subsubsection{Examples of ``Best'' and ``Worst'' Hospitals} 
 
 Row sum ranking is a poor metric for differentiation between most hospitals, but it is able to distinguish the extremal hospitals.  In other words, hospitals with almost all features considerably above the mean would have obviously been assigned a 10 within any ranking system, and hospitals with almost all features below the mean would have obviously been assigned a 1 by any ranking system. 
 
 Let us take the $25$ best and worst hospitals under the row sum ranking, and consider how well these correspond to the final DNN rankings from Section \ref{DNNStars}.  The average quality function value across the top $25$ row sum rankings is 8.03, and the average quality function value across the bottom $25$ row sum rankings is 1.75.

 While these averages are indicative of the agreement between the DNN ranking and simple understandings of hospital quality, they are not perfect.  For example, the top $25$ row sum hospitals have an average DNN ranking of 8.03 because one of these hospitals has a DNN ranking of 4.84.  This hospital 
 is an interesting profile.  Qualitatively, its process features and survey features are above average, but their outcomes are significantly worse than the mean.  Even though only 16 of its 81 features are below the mean, the features for mortality from heart failure and pneumonia, and several measures of hospital associated infection, are all more than 1 standard deviation below the mean.   As all five of those features are heavily weighted by the algorithm and deemed crucial by experts, the hospital is assigned a low ranking despite having a strong row sum ranking.
  

 Figure \ref{fig:simpleMetricEmbedding} shows the embeddings of hospitals without the neural net step.  In other words, we simply weight the mortality and readmission features and calculate an affinity based off those feature vectors.  There are two important details.  First, the embedding is very unstructured, implying a lack of structure within the naive affinity matrix.  Second, while the cloud doesn't necessarily reflect structure, it does give a simple understanding of local neighborhoods.  
 
 Figure \ref{fig:simpleMetricEmbedding} also shows that our new notion of rankings reflects a naive understanding of hospital quality.  Namely, local Euclidean neighborhoods maintain the same ranking.  Also, the 10 best and worst hospitals under the row sum ranking are circled.  As expected, they exist on the two extremes of the embedding cloud, and our rankings agree with their notion of quality.
  
 \begin{figure}[h!]
 \begin{tabular}{cc}
 \includegraphics[width=.45\textwidth]{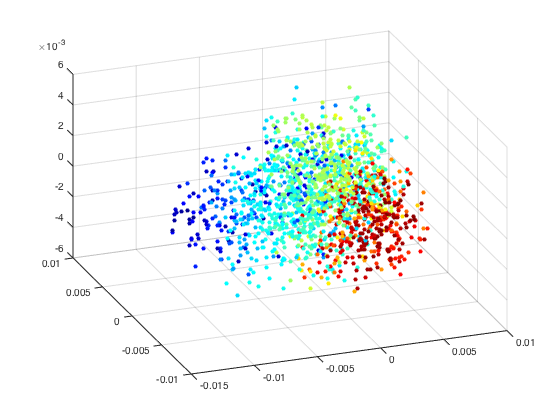} & 
  \includegraphics[width=.45\textwidth]{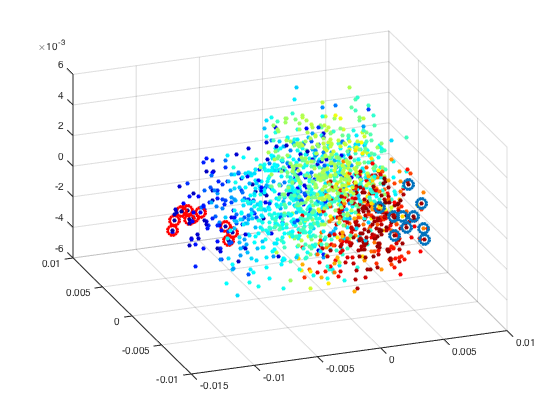}
 \end{tabular}
 \caption{Left: Embedding of hospitals using Euclidean distance on profiles, with $2\times$ weighting on morality features and $1.5\times$ weighting on readmission features.  Coloring is value of the quality function described in Section \ref{DNNStars}, with top ratings being red and bottom ratings being blue.  Right: Same embedding, with extremal points circled.  Blue circles on $10$ hospitals with largest number of features above the mean, and red circles on $10$ hospitals with fewest number of features above the mean.  This is a rough characterization of ``best'' and ``worst'' hospitals.}\label{fig:simpleMetricEmbedding}
 \end{figure}

 \subsubsection{Dependence on Initial Rankings}\label{initialRankings}
 
Another important feature of a ranking algorithm is its robustness across multiple runs.  To demonstrate stability of the neural net section of the algorithm, we run multiple experiments with random parameters to test the stability of the quality function.  Figure \ref{fig:starsAcrossRuns} shows the quality rating from one run of the algorithm (ie. the average ranking after 100 iterations of the neural net) against the quality rating from a separate run of the algorithm.  In both runs, $1,000$ hospitals from $\Omega$ are randomly chosen for training the model and then tested on the other $614$ hospitals in $\Omega$.  Figure \ref{fig:starsAcrossRuns} shows the rankings for all $1,614$ hospitals across both runs. Given the strong similarity in both quality rating and overall organization, it is clear that the average over 100 iterations of a neural net is sufficient to decide the quality function.
 
 \begin{figure}
 \begin{tabular}{cc}
 \includegraphics[width=.5\textwidth]{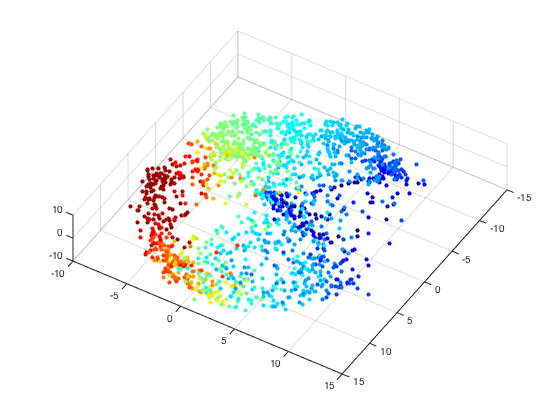} & 
  \includegraphics[width=.5\textwidth]{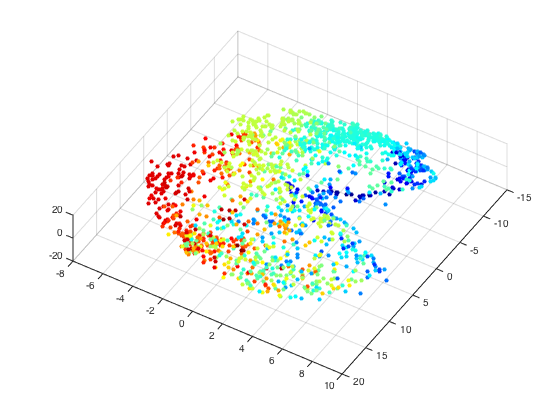}
 \end{tabular}
 \caption{(left) Original embedding colored by quality function.  Same image as Figure \ref{fig:subsetEmbedding}.  (right) New embedding of 100 nets trained independently from first run.  Embedding is finally rotated to match orientation of the original via \cite{Hirn}.  Quality rating coloring the second embedding is generated independently of left embedding.}\label{fig:starsAcrossRuns}
 \end{figure}
 
 It is also important to examine the dependence our initial binning and ranking has on the final ranking of the hospitals.  Clearly the initial binning is only meant to give approximate ranks, so a strong dependence on these rankings would be problematic.  Table \ref{tab:confusionRankings} shows the confusion matrix between the initial hospital rankings and the final rankings assigned from the neural net.  The purpose of the neural net second step is to reclassify hospitals that are binned incorrectly due to spurious correlations.

      \begin{table}
      \begin{center}
 \begin{tabular}{|c|cccc|}
 \hline
 &  {\bf Final} & {\bf Rank} & & \\
 \hline
 
  &    8  &  11 &    0  &   0\\
  {\bf First Rank} & 15 &  730  & 129  &   0\\
   &  0  &  97 &  330  & 219\\
   &  0  &   0   & 27 &   48\\
     \hline
     \end{tabular}
     \caption{Confusion matrix between initial ranking on hospital bins and final ranking from neural net.  For simplicity, rankings have been rounded into quartiles for purposes of confusion matrix.}\label{tab:confusionRankings}
     \end{center}
     \end{table}

 Let us examine a few of the hospitals that made the biggest jump.  One of the hospitals increased from a rating of 3 to a DNN rating of 6.55.  
 Several of its process features are well below average with a couple 2 standard deviation below average, and its survey features are average to slightly below average.  For this reason, it was immediately classified as a poor hospital.  However, its readmission features across the board are above average, with readmission due to heart failure 1.5 standard deviations above average.  Also, mortality due to heart failure is a standard deviation above average, and all features of hospital associated infection are above the mean as well.  In the initial ranking of the pseudohospitals, the fact that there are 29 survey features as compared to 5 readmission features brought the weighting down, even when giving extra weight to the readmission features.  After the neural net ranking, the ranking is improved to 6.55.


   Another hospital dropped from a rating of 7 to a DNN rating of 3.4.  
 Its process feature are all above average, and its readmission for heart failure is 1.5 standard deviations above the mean.  Because readmission is weighted initially and there are a large number of process measures, the hospital was initially binned with better hospitals.  However, its features for hospital associated infections, as well as mortality due to pneumonia and heart failure are a standard deviation below the mean, and its average survey score is a standard deviation below the mean as well.  For these reasons, the DNN ranking dropped significantly.


 \subsubsection{Affinity Dependence on Individual Features}\label{prongsOfEmbedding}

As another validation mechanism, we consider the smoothness of each feature against the metric $\rho$.  For each feature $f_{i}$, we estimate the Lipschitz constant $L_{i}$ such that 
\begin{eqnarray*}
|f_{i}(x) - f_{i}(y)| \le L_{i} \|x - y\|_{DNN},
\end{eqnarray*}
where $\|\cdot\|_{DNN}$ is simply the Euclidean distance in the diffusion embedding space of the DNN affinity matrix $A$.  Once $f$ is normalized by range$(f)$, $L_{i}$ gives a measure of the smoothness of $f_{i}$ under our new metric.

%

\newpage
\centering

Lipschitz Constant Under $\rho$

\begin{minipage}[t]{0.4\textwidth}
\footnotesize
\begin{tabular}{|c|c|}
\hline
{\bf Feature} & {\bf Lipschitz Constant} \\
\hline
Quality Rating Function & 0.124\\
READM 30 HOSP WIDE & 0.303\\
MORT 30 HF & 0.346\\
READM 30 HF & 0.350\\
MORT 30 AMI & 0.371\\
MORT 30 PN & 0.384\\
READM 30 PN & 0.398\\
READM 30 AMI & 0.430\\
READM 30 HIP KNEE & 0.440\\
H HSP RATING 0 6 & 0.582\\
H RECMND DN & 0.595\\
H COMP 1 SN & 0.618\\
H HSP RATING 9 10 & 0.620\\
H COMP 3 SN & 0.688\\
H RECMND DY & 0.691\\
H COMP 1 A & 0.712\\
H QUIET HSP SN & 0.735\\
VTE 5 & 0.740\\
PSI 90 SAFETY & 0.758\\
H COMP 2 SN & 0.758\\
H COMP 4 SN & 0.761\\
COMP HIP KNEE & 0.762\\
STK 8 & 0.763\\
ED 1b & 0.766\\
OP 20 & 0.789\\
ED 2b & 0.795\\
PSI 4 SURG COMP & 0.807\\
VTE 3 & 0.807\\
H RECMND PY & 0.811\\
H HSP RATING 7 8 & 0.831\\
OP 21 & 0.832\\
STK 1 & 0.834\\
HAI 5 & 0.838\\
H CLEAN HSP SN & 0.839\\
H COMP 5 SN & 0.844\\
HAI 2 & 0.846\\
HAI 3 & 0.868\\
H COMP 6 Y & 0.870\\
H COMP 6 N & 0.870\\
HAI 1 & 0.875\\
STK 6 & 0.889\\
H COMP 4 A & 0.898\\
PC 01 & 0.904\\
VTE 2 & 0.906\\
\hline
\end{tabular}
\end{minipage}\hspace{\fill}\begin{minipage}[t]{0.4\textwidth}
\footnotesize
\begin{tabular}{|c|c|}
\hline
{\bf Feature} & {\bf Lipschitz Constant} \\
\hline
H COMP 1 U & 0.908\\
OP 11 & 0.913\\
OP 9 & 0.913\\
H COMP 3 A & 0.918\\
H COMP 2 A & 0.931\\
OP 13 & 0.934\\
OP 18b & 0.935\\
STK 10 & 0.942\\
H QUIET HSP A & 0.946\\
OP 10 & 0.978\\
H COMP 5 A & 0.996\\
H CLEAN HSP A & 1.023\\
HF 3 & 1.077\\
AMI 2 & 1.081\\
HAI 6 & 1.088\\
SCIP INF 9 & 1.102\\
VTE 4 & 1.106\\
AMI 10 & 1.133\\
OP 6 & 1.147\\
H COMP 2 U & 1.186\\
SCIP CAR & 1.192\\
VTE 1 & 1.221\\
H CLEAN HSP U & 1.225\\
OP 22 & 1.227\\
H QUIET HSP U & 1.258\\
SPP & 1.264\\
IMM 1A & 1.272\\
IMM 2 & 1.284\\
STK 2 & 1.288\\
OP 7 & 1.291\\
H COMP 4 U & 1.325\\
H COMP 3 U & 1.365\\
HF 1 & 1.430\\
STK 5 & 1.443\\
PN 3b & 1.545\\
H COMP 5 U & 1.566\\
SCIP INF 2 & 1.994\\
SCIP INF 3 & 2.001\\
PN 6 & 2.066\\
SCIP INF 1 & 2.131\\
SCIP VTE & 2.151\\
SCIP INF 10 & 3.333\\
HF 2 & 5.111\\
& \\
\hline

\end{tabular}
\end{minipage}

\flushleft
From the Lipschitz constant, we can conclude that the most influential features are the mortality and readmission scores, followed closely by survey scores and post surgical infection scores.  While the process features have some smoothness, the Lipschitz constants are significantly larger than outcome and some survey scores.

\begin{figure}[h!]
\begin{center}
\includegraphics[width=.45\textwidth]{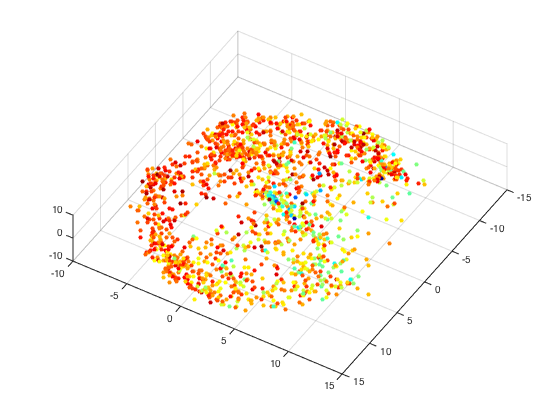}
\end{center}
\caption{Hospital Diffusion embedding colored by average across 34 process features.}\label{fig:starVsProcess}
\end{figure}

Almost all the hospitals in Figure \ref{fig:starVsProcess} have an average process score within a half standard deviation of the mean, which is well within acceptable levels.  For this reason, these features are not as strong of features as some of the other features.  However, the few hospitals that are well below the mean have middle to poor quality ratings. 

Figures \ref{fig:starVsSurvey} and \ref{fig:starVsOutcome} show the embedding of the hospitals colored by their average survey and average outcome features.  These are more strongly correlated with the overall quality function.
\begin{figure}[h!]
\begin{center}
\includegraphics[width=.45\textwidth]{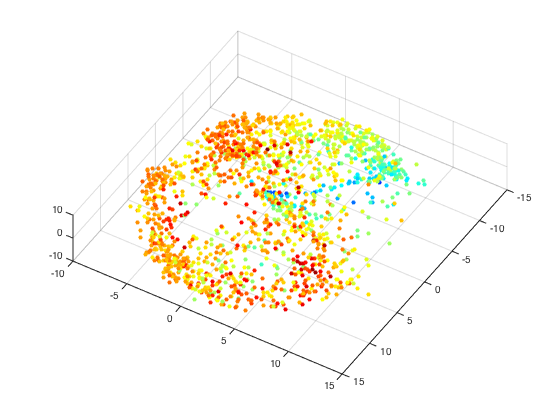}
\end{center}
\caption{Hospital Diffusion embedding colored by average across 29 survey features.}\label{fig:starVsSurvey}
\end{figure}

\begin{figure}[h!]
\begin{center}
\includegraphics[width=.45\textwidth]{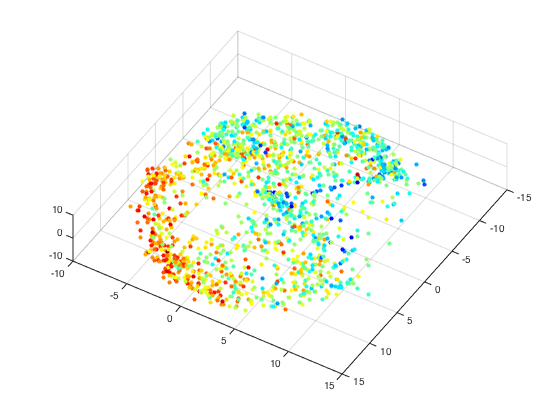}
\end{center}
\caption{Hospital Diffusion embedding colored by average across 15 outcome features.}\label{fig:starVsOutcome}
\end{figure}

Note that the average across all features from a category is not the ideal description of that category.  For a given hospital, there can be significant variability within a category which would not be captured by the mean.  The category average is only meant to demonstrate a general trend.  Figure \ref{fig:starVsMortReadLS} shows a weighted average of the readmission and mortality features, with the weights determined by a least squares fit with the DNN quality function as a dependent variable.

\begin{figure}[h!]
\begin{tabular}{cc}
\includegraphics[width=.45\textwidth]{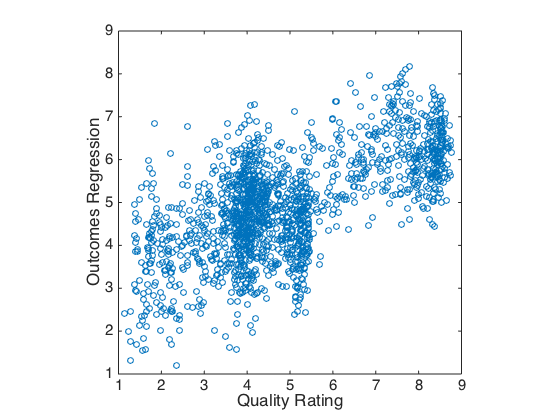} & 
\includegraphics[width=.45\textwidth]{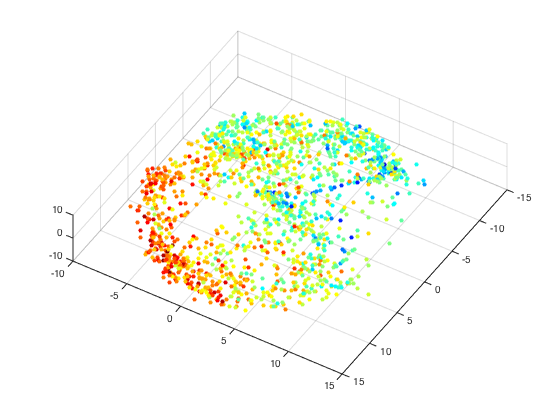}
\end{tabular}
\caption{(left) Scatter plots of quality function vs least squares solution with mortality and readmission as independent variables and quality function as dependent variable.  (right) Hospital Diffusion embedding colored by same least squares solution.}\label{fig:starVsMortReadLS}
\end{figure}

We can also use these embeddings, along with Figure \ref{fig:readMortEmbedding}, to show why the embedding has the two prongs on the left side of the image among the well ranked hospitals.  Notice from Figure \ref{fig:readMortEmbedding} that readmission is very strong along the lower prong, while mortality scores are very strong along the upper prong.  Moreover, there is some level of disjointness between the two scores \cite{harlanReadMort}.
\begin{figure}[h!]
\begin{tabular}{cc}
\includegraphics[width=.45\textwidth]{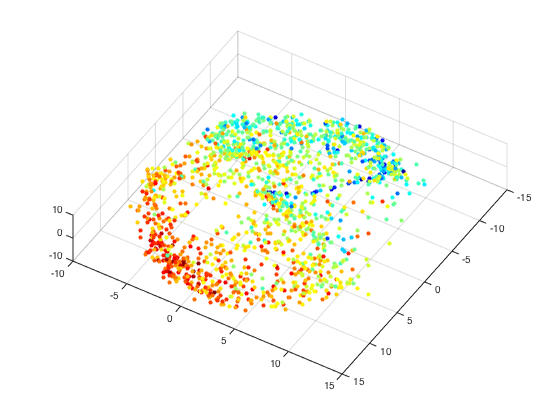} & 
\includegraphics[width=.45\textwidth]{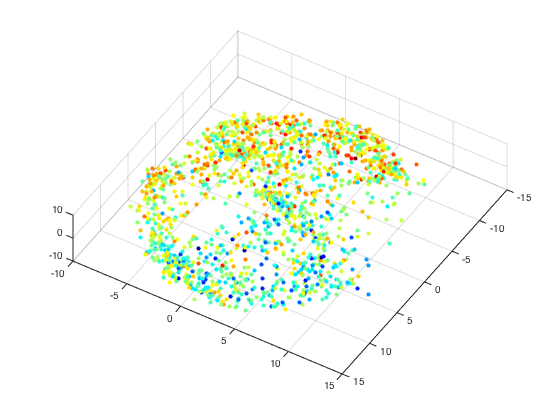}
\end{tabular}
\caption{(left) Hospital Diffusion embedding colored by hospital wide readmission  (right) Hospital Diffusion embedding colored by mortality from heart failure.}\label{fig:readMortEmbedding}
\end{figure}

This begs the question of why the strong readmission cluster is ranked higher than the strong mortality cluster.  That can be explained in Figure \ref{fig:starVsSurvey}.  The survey features are much stronger along the lower cluster with high readmission than they are in the mortality cluster.  This explains the cluster of hospitals receiving top quality ratings despite having average to slightly above average mortality scores.  The same can be said, to a lesser degree, for the process scores in Figure \ref{fig:starVsProcess}.  So the differentiation in rankings is derived from the fact that, at their highest levels, more of the features agree with readmission features than they do with mortality features.

The left plot in Figure \ref{fig:meanScores} represents each hospital with three features: the mean score across process features, the mean score across survey features, and the mean score across outcome features.  The plot is colored by the quality function, and it is clear that the rankings reflect the trend of these three average features.  The right plot in Figure \ref{fig:meanScores} represents each hospital with the NNLS process, survey, and outcome scores.  Table \ref{tab:strongNNLSMetrics} list those features that have non-trivial weights with p-value $<.05$.

\begin{figure}[h!]
\begin{tabular}{cc}
\includegraphics[width=.45\textwidth]{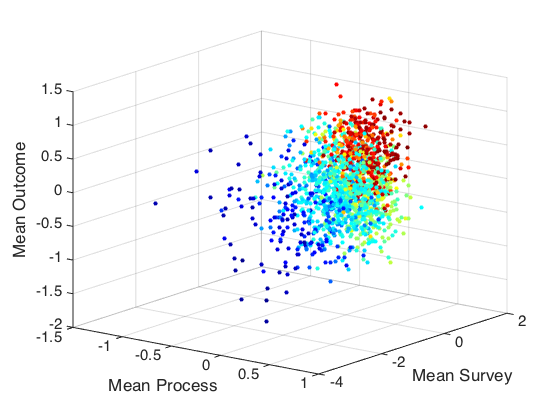} & 
\includegraphics[width=.45\textwidth]{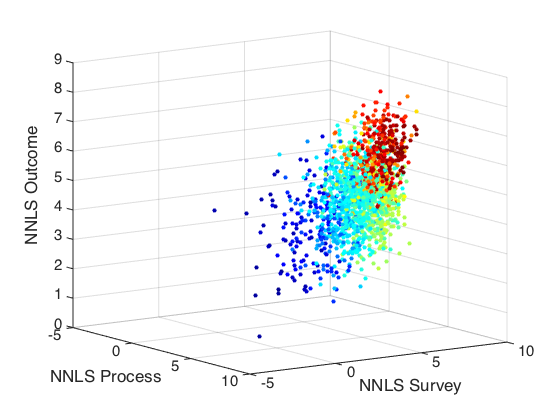}
\end{tabular}
\caption{(left) Plot of mean scores across process, survey, and outcome measures, colored by the quality function.  (right) Plot of process, survey, and outcome measures with weights determined by non negative least squares with the quality function as dependent variable, colored by quality rating.}\label{fig:meanScores}
\end{figure}

\begin{table}[ht!]
\footnotesize
\begin{centering}
\begin{tabular}{|c|c|}
\hline
{\bf Features} & {\bf Coefficient Magnitude}\\
\hline
HAI 3 & 0.029\\
HAI 6 & 0.010\\
READM 30 AMI & 0.148\\
READM 30 HF & 0.139\\
READM 30 HIP KNEE & 0.148\\
READM 30 HOSP WIDE & 0.122\\
READM 30 PN & 0.074\\
PSI 4 SURG COMP & 0.018\\
PSI 90 SAFETY & 0.003\\
MORT 30 HF & 0.116\\
MORT 30 PN & 0.121\\
MORT 30 AMI & 0.072\\
SPP & 0.053\\
OP 10 & 0.045\\
OP 11 & 0.080\\
OP 13 & 0.003\\
AMI 10 & 0.069\\
HF 1 & 0.058\\
HF 3 & 0.136\\
STK 1 & 0.136\\
STK 5 & 0.040\\
STK 6 & 0.156\\
STK 8 & 0.059\\
STK 10 & 0.083\\
VTE 2 & 0.183\\
VTE 3 & 0.126\\
VTE 4 & 0.002\\
PN 6 & 0.112\\
SCIP CAR & 0.074\\
IMM 2 & 0.239\\
OP 6 & 0.077\\
OP 7 & 0.009\\
OP 21 & 0.028\\
PC 01 & 0.108\\
H CLEAN HSP SN & 0.171\\
H COMP 1 SN & 0.141\\
H COMP 2 SN & 0.082\\
H COMP 5 SN & 0.119\\
H COMP 6 Y & 0.156\\
H HSP RATING 7 8 & 0.185\\
H RECMND DN & 0.176\\
H RECMND PY & 0.123\\

 \hline
 \end{tabular}
 \caption{Significant features from Non Negative Least Squares with Quality Function as the Dependent Variable}\label{tab:strongNNLSMetrics}
\end{centering}
\end{table}

 \section{Conclusion}
 
 We introduced an algorithm for generating new metrics and diffusion embeddings based off of expert ranking.  Our algorithm incorporates both data point geometry via hierarchical diffusion geometry and non-linear meta-features via stacked neural nets.  The resulting embedding and rankings represent a propagation of the expert rankings to all data points, and the resulting metric generated by the stacked neural net gives a Lipschitz representation with respect to Euclidean distance that learns important and irrelevant features according to the expert opinions and in automated fashion.
 
 Although the ranking algorithm seems tied to the process of expert rankings of hospitals, the underlying idea of generating metrics in the form of \eqref{eq:rho}. and propagating first pass rankings to the rest of the data, is quite general.  For example, this method could be used to propagate sparsely labeled data to the rest of the dataset, with expert bin rankings replaced by the mode of the labelled data in the bin.

 This method also touches on the importance of incorporating data point organization into the neural net framework when dealing with smaller data sets and noisy features.  Without the expert driven function to roughly organize the data points, the stacked autoencoder fails to determine any relevant features for separation, as shown in Figure \ref{fig:saeNoBackprop}.  
 
 We will examine the implications of our hospital ratings on health policy, as well as discuss the various types of hospitals in our embedding, in a future paper.  Future work will also examine further examination of the influence of data point organization on neural nets and the generation of meta-features.  Also, it would be interesting to examine other applications of propagation of qualitative rankings and measures.

\section*{Acknowledgments}
The authors would like Arjun K. Venkatesh MD, MBA, MHS and Elizabeth E. Drye MD, SM for helping to develop the initial rankings of the pseudohospitals, Uri Shaham for use of his deep learning code, and Ali Haddad for providing the base code for the questionnaire.  Alexander Cloninger is supported by NSF Award No. DMS-1402254.

\newpage

\bibliographystyle{plain.bst}	
\bibliography{Bibliography.bib}		

\end{document}